\documentclass{article}

\usepackage[preprint]{corl_2026} 
\usepackage{microtype}
\usepackage{graphicx}
\usepackage{subcaption}
\usepackage{booktabs}
\usepackage{tabularx}
\usepackage{array}
\usepackage{amsmath}
\usepackage{amssymb}
\usepackage{mathtools}
\usepackage{amsthm}
\usepackage{enumitem}
\usepackage{makecell}
\usepackage{xcolor}
\usepackage{cuted}
\usepackage{multirow}
\usepackage[capitalize,noabbrev]{cleveref}

\hypersetup{
  pdftitle={Graph-Guided Safe Diffuser: Topological Graph Guidance for Safe Diffusion Planning},
  pdfauthor={Nakgyu Yang, KwangBin Lee, SooJean Han},
  pdfsubject={arXiv preprint}
}

\newtheorem{theorem}{Theorem}
\newtheorem{assumption}[theorem]{Assumption}
\newtheorem{lemma}[theorem]{Lemma}

\title{Graph-Guided Safe Diffuser: Topological Graph Guidance for Safe Diffusion Planning}

\author{
  Nakgyu Yang$^{*}$ \quad KwangBin Lee$^{*}$ \quad SooJean Han \\
  School of Electrical Engineering, Korea Advanced Institute of Science and Technology (KAIST) \\
  Daejeon, Republic of Korea \\
  $^{*}$Equal contribution.
}

\begin{document}
\maketitle


\begin{abstract}
Many diffusion-based planners enforce safety through inference-time guidance, but such interleaved trajectory deformations often degrade kinematic feasibility due to manifold rupture.
We propose \textbf{Graph-Guided Safe Diffuser (G2SD)}, a hierarchical framework that leverages a high-level topological graph planner to guide a low-level diffusion model. G2SD enforces safety at a structural level by abstracting the data manifold into a learned latent graph, on which high-level planning is performed. Continuous trajectories are generated by diffusion planners, which are conditioned on the graph node representations selected by the high-level planner. 
Theoretical analyses demonstrate conditions under which manifold rupture occurs in diffusion planners, and show that G2SD improves safety by reducing the constraint violation probability as the number of segments increases.
Experiments demonstrate that G2SD substantially outperforms baselines, increasing goal-reaching rate without any collision from 40-50\% to 98\% in Maze2D navigation and also achieving superior task scores in locomotion.
\end{abstract}




\section{Introduction}
Diffusion probabilistic models have gained significant traction in robotic planning \citep{janner2022diffuser, wang2022diffusion}, enabling the high-fidelity synthesis of high-dimensional, multimodal behaviors \citep{chi2025diffusion}. However, real-world deployment requires satisfying explicit safety constraints (e.g., collision avoidance) often missing from training. Recently, methods for enforcing safety via inference-time guidance in diffusion planning have emerged, using control-theoretic tools like control barrier functions to deform trajectories via gradient updates~\citep{xiao2023safediffuser}. However, enforcing such safety interventions during sampling can degrade generation quality due to conflicts with the learned dynamics throughout the denoising process.

This paper proposes \textbf{Graph-Guided Safe Diffuser (G2SD)}, a hierarchical framework that couples safety-aware topological routing with data-supported local diffusion planning. Offline, G2SD learns a discrete latent transition graph from demonstrations leveraging VQ-VAE~\citep{oord2017vqvae}: nodes represent prototypical states on the kinematic manifold, and edges encode empirically observed transitions. We prune nodes according to explicit safety constraints and solve a weighted shortest-path problem to obtain a constraint-compliant graph-node sequence. A low-level diffusion planner generates conditioned local segments supported by robot trajectory data, reducing long-horizon drift while preserving kinematic and dynamic consistency for safer and more reliable planning.

We support G2SD by theoretically and empirically analyzing the tension between safety guidance and trajectory feasibility. We formally characterize the phenomenon of \textit{manifold rupture}, where strong safety gradients push trajectories away from the data manifold (causing kinematic infeasibility). We then derive a probabilistic safety-violation bound, proving that decomposing long-horizon generation into conditioned segments yields an exponentially tighter bound on the probability of exceeding the local safety margin. Empirically, across long-horizon navigation and locomotion benchmarks, G2SD substantially improves success rates and safety margins over other existing guided safe diffusion baselines.


\section{Related Work}
\label{sec:related}

\textbf{Diffusion-based planning.}
Diffusion models have recently been used as trajectory generators for offline planning and motion generation, including Diffuser~\citep{janner2022diffuser}, trajectory-level navigation planners~\citep{yu2024trajectory}, motion-prior samplers~\citep{carvalho2025motion}, and C-space path generators~\citep{seo2025presto}. These methods demonstrate the ability of diffusion models to produce coherent multimodal trajectories, but they focus on generating feasible plans from learned trajectory priors rather than enforcing safety constraints during sampling.

\textbf{Safety-constrained generative planning.}
Safety in generative planners is often enforced through projection, classifier/gradient guidance, or control-theoretic corrections during sampling~\citep{dhariwal2021diffusion}. Recent methods incorporate control barrier functions, Lyapunov certificates, or related safety filters into diffusion or flow-matching planners to improve constraint satisfaction~\citep{xiao2023safediffuser,yang2025safeflowmatcher,mizuta2024cobl,chengsafe}. While effective for reducing violations, such inference-time corrections can push samples away from the data-supported trajectory manifold, leading to dynamically implausible plans over long horizons.

\textbf{Hierarchical and graph-based planning.}
Hierarchical planning decomposes long-horizon tasks into global decisions and local executions, as in task-and-motion planning and skill stitching~\citep{kaelbling2011hpn,garrett2021tamp,seker2019cnmp}. Graph-based methods, including PRM/RRT, topological maps, learned roadmaps, and planning-over-data, similarly use connectivity structures to support long-horizon reachability~\citep{lavalle2001rrt,Dijkstra1959,blochliger2018topomap,qureshi2019mpnet,saroya2021roadmap,eysenbach2019sorb}. Our method combines these ideas with diffusion planning: it learns a data-derived latent topological graph from demonstrated trajectories, prunes unsafe nodes before search, and uses diffusion only for short conditioned bridges.


\section{Problem Formulation}\label{sec:problem_formulation}

We consider long-horizon robotic trajectory generation with dynamic feasibility and explicit safety constraints.
Let $\boldsymbol{\tau}=\{\mathbf{s}_h\}_{h=0}^{H-1}\in\mathbb{R}^{H\times D}$ denote a trajectory of horizon $H$, where $\mathbf{s}_h\in\mathbb{R}^D$ is the robot state. 
We distinguish three geometric objects in trajectory space.
The \textit{ambient space} $\mathbb{R}^{H\times D}$ contains all trajectory parameters.
The \textit{kinematic manifold} $\mathcal{M}_{\mathrm{kin}}\subset\mathbb{R}^{H\times D}$ contains trajectories satisfying robot kinematic and dynamic constraints, such as joint limits and temporal continuity.
The \textit{data manifold} $\mathcal{M}\subset\mathcal{M}_{\mathrm{kin}}$ denotes the behaviorally relevant subset supported by expert demonstrations.
We formalize this local geometric view as follows.

\begin{assumption}[Data Manifold]
\label{assum:data_manifold}
There exists a smooth embedded submanifold
$\mathcal{M}\subset\mathcal{M}_{\mathrm{kin}}$ such that the expert data distribution $p_{\mathrm{data}}$ is concentrated near $\mathcal{M}$:
\begin{equation*}
    \mathbb{P}_{\boldsymbol{\tau}\sim p_{\mathrm{data}}}
    \bigl(d_{\mathcal{M}}(\boldsymbol{\tau})\le \delta\bigr)
    \ge 1-\epsilon ,
\end{equation*}
where
$d_{\mathcal{M}}(\boldsymbol{\tau})
\coloneqq
\inf_{\mathbf{y}\in\mathcal{M}}\|\boldsymbol{\tau}-\mathbf{y}\|$,
$\delta>0$ is small relative to the data scale, and $\epsilon>0$ is a probability tolerance.
\end{assumption}

For multimodal or hybrid data, this assumption can be interpreted locally on each smooth component, which is sufficient for our subsequent analysis.

Diffusion planners generate trajectories by reversing a stochastic denoising process.
Starting from a Gaussian prior $\boldsymbol{\tau}_T\sim\mathcal{N}(0,I)$ in the ambient space, a learned score
$\mathbf{s}_{\theta}(\boldsymbol{\tau}_t,t)\approx\nabla_{\boldsymbol{\tau}}\log p_t(\boldsymbol{\tau}_t)$
iteratively denoises toward a feasible trajectory.
A standard discretized update is
\begin{equation}
    \boldsymbol{\tau}_{t-1}
    \leftarrow
    \boldsymbol{\tau}_t
    + \eta\,\underbrace{\mathbf{s}_{\theta}(\boldsymbol{\tau}_t,t)}_{\text{prior drift}}
    + \sigma_t \mathbf{z},
    \quad
    \mathbf{z}\sim\mathcal{N}(0,I),
    \label{eq:score_update}
\end{equation}
where $\sigma_t$ is a noise schedule.
The prior drift encourages samples to shift near the learned data manifold $\mathcal{M}$.

Safety-guided planners that enforce safety, such as obstacle avoidance, introduce an auxiliary cost
$\mathcal{J}_{\mathrm{safe}}(\boldsymbol{\tau}_t)$ and modify the score by adding a guidance force:
\begin{equation}
    \widehat{\mathbf{s}}_{\theta}(\boldsymbol{\tau}_t,t)
    =
    \mathbf{s}_{\theta}(\boldsymbol{\tau}_t,t)
    -
    \gamma\,
    \underbrace{\nabla_{\boldsymbol{\tau}_t}
    \mathcal{J}_{\mathrm{safe}}(\boldsymbol{\tau}_t)}_{\text{safety force}},
    \label{eq:guided_score}
\end{equation}
where $\gamma$ is the guidance scale. This construction assumes that the safety force nudges trajectories into the safe region without disrupting the learned prior. However, strong safety gradients can interfere with the learned denoising dynamics and push samples away from the data manifold. We refer to this failure as \textit{manifold rupture}. In this regime, generated trajectories may satisfy local safety constraints while becoming kinematically infeasible. Section~\ref{sec:rupture_theory} elaborates this effect in more detail.

\section{Graph-Guided Safe Diffuser (G2SD)}\label{sec:g2sd_main}

We propose \textbf{Graph-Guided Safe Diffuser (G2SD)}, a framework that unifies offline latent topological structure with online safe generative control. Our approach incorporates safety at the structural planning level through two mechanisms. First, offline graph construction builds a discrete graph from the trajectory dataset using VQ-VAE state discretization, where nodes represent prototypical robot states and edges encode data-supported transitions. Second, online hierarchical planning supports safety by pruning risky nodes that violate task-specific safety constraints, followed by high-level node planning to compute a constraint-compliant graph-node sequence. A low-level diffusion planner generates locally feasible segments between successive grounded graph nodes. These segments are concatenated to form the complete start-to-goal trajectory.

\subsection{Offline Graph Construction} 
\label{sec:graph_construct}
We construct an offline graph \(\mathcal{G}_{\text{total}}=(\mathcal{Z},\mathcal{E})\) from expert trajectories as a discrete abstraction of empirical reachability on the state-level projection of the demonstrated trajectory manifold \(\mathcal{M}\). Nodes $\mathcal{Z}$ provide a discrete abstraction of states, and directed edges $\mathcal{E}$ encode empirically observed transitions.

\textbf{VQ-VAE Discretization.}
We learn a discrete abstraction of this induced state-level support from expert demonstrations. We train a VQ-VAE~\citep{oord2017vqvae} to map each continuous robot state $\mathbf{s}$ to a discrete codebook index. The encoder assigns $E_\phi(\mathbf{s}_h)=z_i$, where $z_i\in\mathcal{Z}:=\{1,\dots,Z\}$. We use these indices as graph nodes, each representing a cluster of prototypical robot states. For each node, we store representative physical states to enable explicit safety checks when pruning unsafe nodes, as described in Section~\ref{subsec:onlineplanning}.

\textbf{Graph Construction.}
We connect the nodes based on empirical reachability over a fixed temporal local horizon $N$. A directed edge $(z_i, z_j)$ is added if the dataset contains a transition pair $(\mathbf{s}_h, \mathbf{s}_{h+N})$, where $E_\phi(\mathbf{s}_h)=z_i$ and $E_\phi(\mathbf{s}_{h+N})=z_j$. During this construction, we also estimate the empirical transition probability $P(z_j | z_i)$ from the observed transition frequencies in the dataset. We model $\mathcal{G}_\text{total}$ as a weighted directed graph by assigning each observed edge a cost. To incorporate both the geometric progress and the reliability of transitions, we define the edge weight $W(z_i, z_j)$ as:
\begin{equation}
    W(z_i, z_j) \coloneqq  \lambda_\text{dist}\underbrace {c(z_i, z_j)}_{\text{geometric cost}} + \lambda_\text{prob}  \underbrace{(-\log P(z_j | z_i))}_{\text{transition cost}},
    \label{eq:edge_weight}
\end{equation}
where $c(\cdot,\cdot)$ is a task-dependent geometric cost function, and $\lambda_\text{dist}, \lambda_\text{prob} {\,\geq\,} 0$ are tunable balancing parameters. Appendix~\ref{app:Cost_Formulation} contains discussion on $c(\cdot,\cdot)$ and the parameters.

\subsection{Online Safe Hierarchical Planning} \label{subsec:onlineplanning}

\textbf{Safety Pruning.}
We obtain the safe subgraph $\mathcal{G}_{\text{safe}} \subseteq \mathcal{G}_{\text{total}}$ by evaluating each node's representative physical state under task-specific safety constraints. We exclude nodes intersecting unsafe regions in navigation and nodes with unstable configurations in locomotion, such as height or pose-limit violations. This yields a pruned graph whose candidate graph-node sequence satisfies the safety checks.

\textbf{High-level Node Planning.}
The high-level planner identifies an edge-cost-optimal graph-node sequence by solving a shortest-path problem (e.g., via Dijkstra's algorithm) on $\mathcal{G}_{\text{safe}}$. Formally, we map the given states $\mathbf{s}^{\text{start}}$ and $\mathbf{s}^{\text{goal}}$ to their nearest graph nodes $z_0$ and $z_M$, and search for a discrete path through the safe subgraph. Specifically, we compute the sequence $\boldsymbol{\tau}_{\text{graph}}^* {\,=\,} (z_0, z_1, \dots, z_M)$ of variable length $M$ that minimizes the cumulative edge weight.
\begin{equation}
    \boldsymbol{\tau}_{\text{graph}}^*
    =
    \operatorname*{argmin}_{(z_0, \dots, z_M)}
    \sum_{k=0}^{M-1} W(z_k, z_{k+1}),
\end{equation}
where \(M\) is the number of low-level trajectory segments, determined by the latent graph topology rather than being fixed a priori.

\textbf{Low-level Diffusion Planning.}
Given the high-level path \(\boldsymbol{\tau}_{\text{graph}}^*=(z_0,z_1,\ldots,z_M)\), each node \(z_k\) in this path is grounded to a representative physical state \(\mathbf{s}_k\), and the low-level planner generates an endpoint-conditioned trajectory segment from \(\mathbf{s}_k\) to \(\mathbf{s}_{k+1}\). Each segment is therefore anchored at both its start and end states, converting long-horizon generation into a sequence of short endpoint-conditioned diffusion bridges. These diffusion bridges sample data-supported intermediate motions from the learned local motion manifold, preserving kinematic and dynamic consistency while reducing off-manifold artifacts. By Theorem~\ref{thm:safety_bound}, these local endpoint conditions yield an exponentially tighter bound on the probability of exceeding the local safety margin.


\section{Theoretical Analysis}
\label{sec:theory}

We provide a geometric explanation for why reactive safety guidance can degrade diffusion planning, and then show how endpoint-anchored generation yields a probabilistic safety bound.

\subsection{Geometric Conflict and Manifold Rupture}
\label{sec:rupture_theory}

We first analyze the conflict between the learned trajectory prior and external safety guidance.

\begin{assumption}[\textbf{Regularity Conditions}]
\label{assump:regularity}
\leavevmode
\begin{enumerate}
    \item $\mathcal{M}$ is a $\mathcal{C}^2$ embedded submanifold of $\mathbb{R}^{H\times D}$ with positive reach $\rho>0$. Every $\boldsymbol{\tau}$ with $d_{\mathcal{M}}(\boldsymbol{\tau})<\rho$ has a unique nearest-point projection onto $\mathcal{M}$.
    \item The learned score $\mathbf{s}_{\theta}(\boldsymbol{\tau},t)$ is $\mathcal{C}^1$ and norm-bounded:
    $\|\mathbf{s}_{\theta}(\boldsymbol{\tau},t)\|\le C_s$ for all $t\in[0,T]$ and all $\boldsymbol{\tau}$ with $d_{\mathcal{M}}(\boldsymbol{\tau})<\rho$.
\end{enumerate}
\end{assumption}

Assumption~\ref{assump:regularity} ensures that, on the tubular neighborhood $\mathcal{N}_{\rho}\coloneqq\{\boldsymbol{\tau}:0<d_{\mathcal{M}}(\boldsymbol{\tau})<\rho\}$, the outward normal $\mathbf{n}(\boldsymbol{\tau})=\nabla d_{\mathcal{M}}(\boldsymbol{\tau})$ is well-defined and the squared-distance function $\frac12 d_{\mathcal{M}}(\boldsymbol{\tau})^2$ is locally $\mathcal{C}^2$.

\begin{theorem}[Manifold Rupture]
\label{thm:rupture}
Consider the Euler--Maruyama update at diffusion step $t$:
\begin{equation}
\label{eq:update_rule_rigorous}
    \boldsymbol{\tau}_{t-1}
    =
    \boldsymbol{\tau}_t
    + \eta \mathbf{s}_{\theta}(\boldsymbol{\tau}_t,t)
    + \sqrt{\eta}\mathbf{z}
    + \boldsymbol{\delta}_{\mathrm{safe}},
    \qquad
    \mathbf{z}\sim\mathcal{N}(0,I),
\end{equation}
where $\boldsymbol{\tau}_t\in\mathcal{N}_{\rho}$ and
$\boldsymbol{\delta}_{\mathrm{safe}}=\eta\mathbf{f}_{\mathrm{safe}}$.
Let $r_{\mathcal{M}}(\boldsymbol{\tau})\coloneqq\frac12 d_{\mathcal{M}}(\boldsymbol{\tau})^2$,
$\mathbf{n}_t\coloneqq\nabla d_{\mathcal{M}}(\boldsymbol{\tau}_t)$, and
$R_{\rho}\coloneqq
\sup_{\boldsymbol{\tau}\in\mathcal{N}_{\rho}}
\left|\operatorname{Tr}(\nabla^2 r_{\mathcal{M}}(\boldsymbol{\tau}))\right|$.
Suppose $\mathbf{n}_t^\top\mathbf{f}_{\mathrm{safe}}>0$ and define
$\mathbf{f}_{\mathrm{safe},\perp}\coloneqq
(\mathbf{n}_t^\top\mathbf{f}_{\mathrm{safe}})\mathbf{n}_t$.
If
\begin{equation}
\label{eq:dominance_condition}
    d_{\mathcal{M}}(\boldsymbol{\tau}_t)
    \bigl(\|\mathbf{f}_{\mathrm{safe},\perp}\|-C_s\bigr)
    -
    \frac12 R_{\rho}>0,
\end{equation}
then, for sufficiently small $\eta$, the expected squared distance to $\mathcal{M}$ strictly increases:
\begin{equation}
\label{eq:bound_squared_distance}
    \mathbb{E}
    \left[
    r_{\mathcal{M}}(\boldsymbol{\tau}_{t-1})
    \mid
    \boldsymbol{\tau}_t
    \right]
    \ge
    r_{\mathcal{M}}(\boldsymbol{\tau}_t)
    +
    \epsilon_{\mathrm{div}},
\end{equation}
where
\begin{equation}
\label{eq:divergence_margin}
    \epsilon_{\mathrm{div}}
    =
    \eta
    \left[
    d_{\mathcal{M}}(\boldsymbol{\tau}_t)
    \bigl(\|\mathbf{f}_{\mathrm{safe},\perp}\|-C_s\bigr)
    -
    \frac12 R_{\rho}
    \right]
    +
    o(\eta)
    >
    0.
\end{equation}
\end{theorem}

\begin{proof}
Let
$\Delta\boldsymbol{\tau}_t
\coloneqq
\eta\mathbf{s}_{\theta}(\boldsymbol{\tau}_t,t)
+\sqrt{\eta}\mathbf{z}
+\eta\mathbf{f}_{\mathrm{safe}}$.
Since $r_{\mathcal{M}}$ is $\mathcal{C}^2$ on $\mathcal{N}_{\rho}$ and
$\nabla r_{\mathcal{M}}(\boldsymbol{\tau})
=d_{\mathcal{M}}(\boldsymbol{\tau})\mathbf{n}(\boldsymbol{\tau})$,
Taylor expansion gives
\begin{equation}
\label{eq:taylor_expansion}
r_{\mathcal{M}}(\boldsymbol{\tau}_{t-1}) = r_{\mathcal{M}}(\boldsymbol{\tau}_t) + \nabla r_{\mathcal{M}}(\boldsymbol{\tau}_t)^\top \Delta\boldsymbol{\tau}_t + \frac12 \Delta\boldsymbol{\tau}_t^\top \nabla^2 r_{\mathcal{M}}(\boldsymbol{\tau}_t) \Delta\boldsymbol{\tau}_t + o(\eta).
\end{equation}
Taking conditional expectation removes the zero-mean linear noise term, while
$\mathbb{E}[\mathbf{z}\mathbf{z}^\top]=I$ contributes
$\frac12\eta\operatorname{Tr}(\nabla^2 r_{\mathcal{M}}(\boldsymbol{\tau}_t))$.
The prior drift is bounded by
\[
\left|
\nabla r_{\mathcal{M}}(\boldsymbol{\tau}_t)^\top
\eta\mathbf{s}_{\theta}(\boldsymbol{\tau}_t,t)
\right|
\le
\eta d_{\mathcal{M}}(\boldsymbol{\tau}_t) C_s,
\]
whereas the safety force contributes
$
\nabla r_{\mathcal{M}}(\boldsymbol{\tau}_t)^\top
\eta\mathbf{f}_{\mathrm{safe}}
=
\eta d_{\mathcal{M}}(\boldsymbol{\tau}_t)
\|\mathbf{f}_{\mathrm{safe},\perp}\|.
$
Substituting into~\eqref{eq:taylor_expansion} yields \eqref{eq:bound_squared_distance}--\eqref{eq:divergence_margin}. Condition~\eqref{eq:dominance_condition} ensures a positive margin for small enough $\eta$.
\end{proof}

Theorem~\ref{thm:rupture} formalizes manifold rupture: a strong safety force can move samples away from the data manifold. G2SD avoids this reactive-guidance failure mode through a safety-pruned latent graph rather than injecting safety corrections into the denoising dynamics. We further provide empirical support for the relevance of~\eqref{eq:dominance_condition} in Appendix~\ref{app:assumption_support}.

\subsection{Probabilistic Safety Bounds from Hierarchical Decomposition}
\label{sec:safety_bounds}

We now analyze the safety benefit of decomposing long-horizon generation into short conditioned segments. G2SD selects a graph path with \(M+1\) nodes, and generates the full trajectory by stitching $M$ conditional diffusion segments.

\begin{assumption}[\textbf{Conditional Independence of Segments}]
\label{assump:independence}
Conditioned on the planned graph-node sequence, the generated segments are independent, and the conditional generation error in each segment is approximated by a Brownian bridge.
\end{assumption}

Appendix~\ref{app:assumption_support} provides implementation and empirical support. Let $\delta_{\mathrm{safe}}$ denote the minimum local safety margin along the planned graph route. We define $\mathrm{Violation}_{\mathrm{G2SD}}$ as the event that the maximum deviation within at least one conditioned segment exceeds this safety margin.

\begin{lemma}[\textbf{Kolmogorov tail bound for Brownian bridges}]
\label{lem:bridge_supremum_cdf}
Let $X_t=\sigma W_t$ be a scaled Brownian motion, and let $B_t$ be the Brownian bridge on $[0,\Delta t]$ obtained by conditioning $X_{\Delta t}=0$. 
Then, for any $\beta>0$,
\begin{equation*}
    \mathbb{P}\!\left(\sup_{t\in[0,\Delta t]} |B_t|>\beta\right)
    \le
    2\exp\!\left(-\frac{2\beta^2}{\sigma^2\Delta t}\right).
\end{equation*}
\end{lemma}

\begin{proof}
For the standard Brownian bridge $\overline{B}$ on $[0,1]$, 
$K=\sup_{t\in[0,1]}|\overline{B}(t)|$ follows the Kolmogorov distribution~\citep{kolmogorov_cdf}:
\[
    \mathbb{P}(K>\beta)
    =
    2\sum_{k=1}^{\infty}(-1)^{k-1}e^{-2k^2\beta^2}
    \le
    2e^{-2\beta^2}.
\]
Rescaling time to $[0,\Delta t]$ and variance to $\sigma^2$ yields the claim.
\end{proof}

\begin{theorem}[\textbf{Violation Probability Bound}]
\label{thm:safety_bound}
Under Assumption~\ref{assump:independence}, consider a trajectory generated by G2SD over horizon $H$ using $M$ endpoint-anchored segments of equal duration $\Delta t=H/M$. Then
\begin{equation}
\label{eq:traj_violation}
    \mathbb{P}\!\left(\mathrm{Violation}_{\mathrm{G2SD}}\right)
    \le
    2M
    \exp\!\left(
    -\frac{2M\delta_{\mathrm{safe}}^2}{\sigma^2 H}
    \right).
\end{equation}
\end{theorem}

\begin{proof}
Let $V_k$ denote the event that segment $k$ violates the safety threshold:
\[
    V_k
    =
    \left\{
    \sup_{t\in[0,\Delta t]} |B_t^{(k)}|
    >
    \delta_{\mathrm{safe}}
    \right\}.
\]
By Lemma~\ref{lem:bridge_supremum_cdf},
\begin{equation}
\label{eq:single_segment_violation}
    \mathbb{P}(V_k)
    \le
    2\exp\!\left(
    -\frac{2\delta_{\mathrm{safe}}^2}{\sigma^2\Delta t}
    \right)
    =
    2\exp\!\left(
    -\frac{2M\delta_{\mathrm{safe}}^2}{\sigma^2H}
    \right).
\end{equation}
Since $\mathrm{Violation}_{\mathrm{G2SD}}=\bigcup_{k=1}^M V_k$, applying the union bound over all segments gives~\eqref{eq:traj_violation}.
\end{proof}

Theorem~\ref{thm:safety_bound} shows that anchoring tightens the safety-violation bound by splitting long-horizon generation into shorter conditional bridges. The resulting exponential term in~\eqref{eq:traj_violation} shows that, in our conditioned setting, shorter segments suppress the probability of unsafe deviations under the bridge approximation.


\section{Numerical Experiments}
\label{sec:exp}

\begin{figure}[t]
  \centering
  \includegraphics[width=0.83\linewidth]{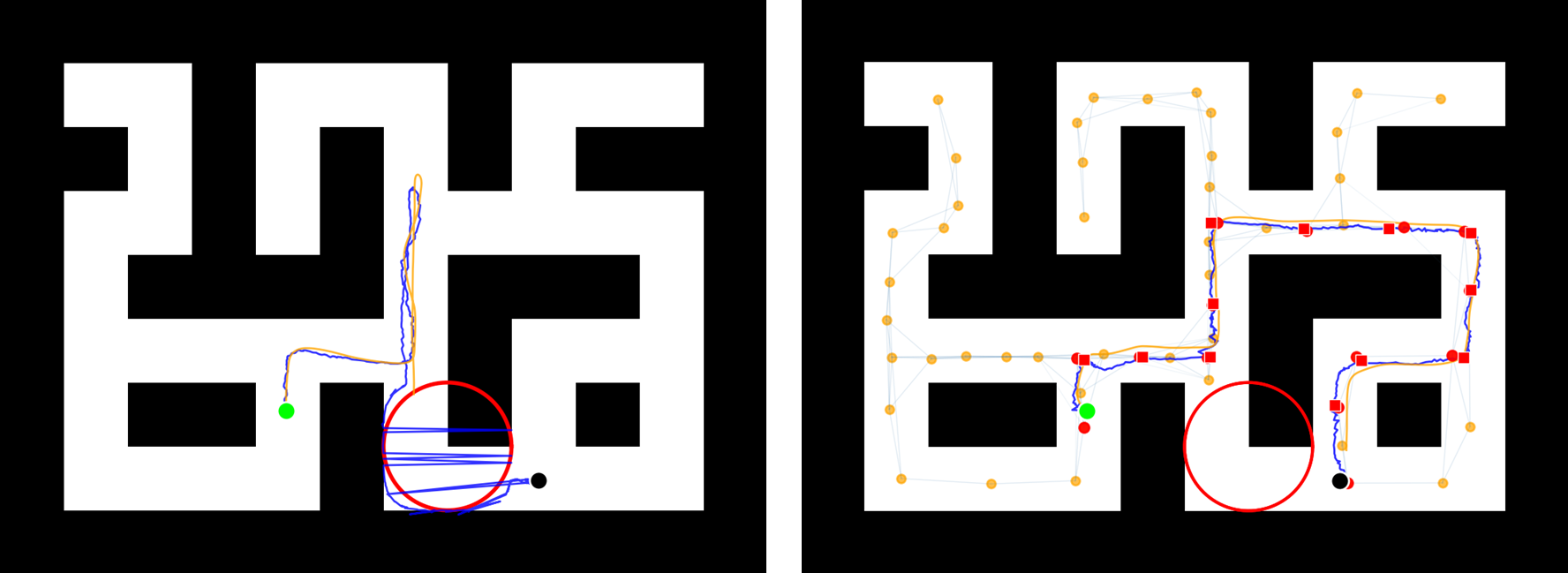}
  \caption{\textbf{SafeDiffuser vs. G2SD in Maze2D.}
  Green and black dots denote start and goal, and the red circle is unsafe.
  \textbf{Left:} SafeDiffuser produces a kinematically infeasible plan that pierces the unsafe region while attempting to avoid both the obstacle and maze walls.
  \textbf{Right:} G2SD plans over the offline latent graph, where orange dots and gray lines denote nodes and edges, and red squares denote selected graph nodes. Conditioning diffusion on these waypoints yields a physically feasible and safe trajectory.}
  \vspace*{1mm}
  \label{fig:compare}
\end{figure}

We benchmark G2SD on three settings: maze navigation (Maze2D-large-v1)~\citep{fu2020d4rl} and robot locomotion (Walker2D-v2, Hopper-v2)~\citep{brockman2016openai}.
We ask:
\textbf{(Q1)} does G2SD avoid unsafe regions during generation, and
\textbf{(Q2)} does it preserve \emph{dynamic feasibility} better than reactive gradient/projection-based safety mechanisms?
We compare against Diffuser~\citep{janner2022diffuser}, Truncated~\citep{brockman2016openai}, Classifier Guidance (CG)~\citep{dhariwal2021diffusion}, SafeDiffuser~\citep{xiao2023safediffuser} using ROS, and SHD~\citep{chen2024simple} using HD-DA.
Each baseline is elaborated in Appendix~\ref{app:baseline} and relevant hyperparameters with computational detail are in Appendix~\ref{app:hyperparameters}.

\subsection{Maze2D Navigation}
\label{sec:exp:maze2d}

\textbf{Setting.}
Maze2D requires an agent to move from start to goal while avoiding maze walls and unsafe regions. Unlike prior fixed start/goal protocols, we randomize start and goal pairs to stress-test performance across varying query distances. All methods are evaluated on 100 episodes with a single unsafe region near a narrow corner.

\textbf{Metrics.}
We report two primary metrics for Q1--Q2. \textit{Safety-SPEC} is the minimum signed distance from the synthesized path to the unsafe-region boundary; negative values indicate penetration, while positive values indicate safety margin. \textit{Success Rate} is the fraction of executed episodes that reach the goal without failure. Failures are decomposed into wall collision ($\text{Fail}_{\text{wall}}$), unsafe-region collision ($\text{Fail}_{\text{obst}}$), and timeout ($\text{Fail}_{\text{time}}$). We additionally report \textit{Trajectory Steps} ($\mu\pm\sigma$) and planning \textit{Time}.

\textbf{Results.}
Table~\ref{tab:maze2d} shows that methods without explicit safety handling (Diffuser/SHD) tend to enter the unsafe region (Safety-SPEC $\approx -0.9$).
Truncated yields near-boundary solutions, and SafeDiffuser achieves only a small positive margin consistent with its QP slack setting (Safety-SPEC $=0.010$).
In contrast, G2SD obtains the largest safety margin (Safety-SPEC $=0.770$), indicating that unsafe-node pruning guides the selected graph route to detour around unsafe regions, rather than relying on local corrections near their boundaries, addressing Q1.
G2SD also achieves the highest success rate, reducing all failure types: wall collisions, unsafe collisions, and timeouts.
This supports Q2, as graph-guided topological planning better preserves dynamic feasibility.
Finally, G2SD adapts its horizon to graph distance, producing shorter paths ($219.5\pm138.9$ steps) than fixed-horizon baselines ($384/372$ steps), while keeping planning time ($0.99$s) comparable to standard diffusion ($0.70$s) and faster than SafeDiffuser ($40.69$s).

\begin{table}[t]
\centering
\small
\setlength{\tabcolsep}{5pt}
\renewcommand{\arraystretch}{1.05}
\begin{tabularx}{0.95\textwidth}{l >{\centering\arraybackslash}X >{\centering\arraybackslash}X c c}
\toprule
\textbf{Method}
& \textbf{Success Rate} ($\uparrow$) \newline {\scriptsize (Fail$_\text{wall}$ / Fail$_\text{obst}$ / Fail$_\text{time}$)}
& \textbf{Safety-SPEC} $(\uparrow)$
& \textbf{Steps} \newline {\scriptsize ($\mu \pm \sigma$)}
& \textbf{Time} {\scriptsize (s)}  ($\downarrow$) \\
\midrule
Diffuser{\scriptsize~\citep{janner2022diffuser}}
& 0.52 {\scriptsize (0.26 / 0.21 / 0.01)}
& $-0.924$
& 384.0 $\pm$ 0.0
& 0.70 \\
Truncated{\scriptsize~\citep{brockman2016openai}}
& 0.37 {\scriptsize (0.30 / 0.33 / 0.00)}
& $-8.941 e^{-7}$
& 384.0 $\pm$ 0.0
& 2.02 \\
Classifier Guidance{\scriptsize~\citep{dhariwal2021diffusion}}
& 0.52 {\scriptsize (0.26 / 0.21 / 0.01)}
& $-0.922$
& 384.0 $\pm$ 0.0
& 0.73 \\
SafeDiffuser{\scriptsize~\citep{xiao2023safediffuser}}
& 0.36 {\scriptsize (0.40 / 0.23 / 0.01)}
& $0.010$
& 384.0 $\pm$ 0.0
& 40.69 \\
SHD{\scriptsize~\citep{chen2024simple}}
& 0.46 {\scriptsize (0.19 / 0.32 / 0.00)}
& $-0.985$
& 372.0 $\pm$ 0.0
& 1.02 \\
\textbf{G2SD (Ours)}
& \textbf{0.98} {\scriptsize (0.02 / 0.00 / 0.00)}
& $\textbf{0.770}$
& 219.5 $\pm$ 138.9
& 0.99 \\
\bottomrule
\end{tabularx}
\vspace{2mm}
\caption{\textbf{Maze2D-large-v1.}
G2SD achieves the best Safety-SPEC and Success Rate, with the lowest wall, unsafe-set, and timeout failures.
Unlike fixed-horizon baselines, it adapts trajectory length to graph distance and remains much faster than baselines.}
\vspace{-2mm}
\label{tab:maze2d}
\end{table}

\subsection{Robot Locomotion}

\textbf{Setting.}
We evaluate two high-dimensional locomotion tasks, Walker2D-v2 and Hopper-v2.
The agent must walk or hop stably without falling, while respecting height limits and avoiding physically impossible motions. Consistent with prior work~\citep{xiao2023safediffuser}, we evaluate both tasks under receding horizon control (RHC), which replans full trajectories for single-action execution. We extend G2SD to online inference with \textit{hybrid diffusion stitching}: diffusion generates the immediate transition to the first subgoal as an online bridge, while the remaining long-horizon tail is retrieved from cached graph edges. Details are in Appendix~\ref{app:inference_details}.

\textbf{Metrics.}
We report Q1--Q2 metrics: \textit{Safety-SPEC}, the minimum safety margin, and \textit{Score}, the normalized cumulative reward. Planning \textit{Time} is reported as a secondary diagnostic.

\textbf{Results.}
Table~\ref{tab:locomotion} summarizes results over 50 episodes.
G2SD achieves the largest planned safety margins on both Walker2D and Hopper, confirming Q1.
Reactive safety baselines such as Truncated and SafeDiffuser concentrate at zero margin, while other methods yield negative margins.
G2SD also obtains the highest Score in both environments, indicating that the latent graph helps satisfy safety requirements while maintaining the learned motion manifold, addressing Q2.
For efficiency, we optimize SHD with parallel batch generation of all subsegments; even against this strengthened hierarchical baseline ($0.91$--$0.92$s), G2SD with hybrid diffusion stitching runs faster ($0.38s$).

\begin{table}[t]
\centering
\small
\vspace{1mm}
\resizebox{0.78\linewidth}{!}{%
\begin{tabular}{@{}llccc@{}}
\toprule
\textbf{Task} & \textbf{Method} & \textbf{Score} ($\uparrow$) & \textbf{Safety-SPEC} ($\uparrow$) & \textbf{Time (s)} ($\downarrow$) \\
\midrule
\textbf{Walker2D}
& Diffuser~\citep{janner2022diffuser} & 0.349 $\pm$ 0.134 & -7.405 & 0.48 \\
& Truncated~\citep{brockman2016openai} & 0.316 $\pm$ 0.100 & 0.0 & 1.72 \\
& Classifier Guidance~\citep{dhariwal2021diffusion} & 0.345 $\pm$ 0.083 & -0.297 & 1.76 \\
& SafeDiffuser~\citep{xiao2023safediffuser} & 0.332 $\pm$ 0.089 & 0.0 & 2.04 \\
& SHD~\citep{chen2024simple} & 0.494 $\pm$ 0.005 & -0.1292 & 0.92 \\
& \textbf{G2SD (Ours)} & \textbf{0.497 $\pm$ 0.007} & \textbf{1.269} & 0.38 \\
\midrule
\textbf{Hopper}
& Diffuser~\citep{janner2022diffuser} & 0.444 $\pm$ 0.058 & -1.597 & 0.52 \\
& Truncated~\citep{brockman2016openai} & 0.438 $\pm$ 0.059 & 0.0 & 1.86 \\
& Classifier Guidance~\citep{dhariwal2021diffusion} & 0.437 $\pm$ 0.086 & -1.842 & 1.78 \\
& SafeDiffuser~\citep{xiao2023safediffuser} & 0.442 $\pm$ 0.065 & 0.0 & 2.10 \\
& SHD~\citep{chen2024simple} & 0.429 $\pm$ 0.040 & -0.127 & 0.91 \\
& \textbf{G2SD (Ours)} & \textbf{0.528 $\pm$ 0.021} & \textbf{0.715} & 0.38 \\
\bottomrule
\end{tabular}%
}
\vspace{2mm}
\caption{\textbf{Walker2D and Hopper.}
 G2SD achieves the best Safety-SPEC and Score, with the fastest inference among safety-aware baselines. Hybrid diffusion stitching also reduces planning time compared to the baselines.}
\label{tab:locomotion}
\vspace{-2mm}
\end{table}

\subsection{Ablation Studies}
\label{sec:ablation}

\textbf{VQ-VAE Usage Weight.} 
We study how $\alpha_{\text{usage}}$ in the VQ-VAE (Section~\ref{sec:graph_construct}) affects effective graph size, measured by \textit{Valid Nodes}, and G2SD performance on Maze2D-large-v1, measured by Success Rate and Safety-SPEC. Valid Nodes counts codebook indices actually selected by the encoder, directly diagnosing codebook collapse: when only a few codes are used, many entries become dead and the latent graph becomes coarse. This collapse can be self-reinforcing, since early small advantages make a few codes more likely to be selected again, concentrating assignments into a subset. Table~\ref{tab:alpha_usage} shows non-monotonic behavior. With $\alpha_{\text{usage}}=0$, codebook collapse limits graph coverage and reduces success. Increasing $\alpha_{\text{usage}}$ grows the number of Valid Nodes, but performance does not monotonically improve, showing that node count alone is insufficient. The drop at $\alpha_{\text{usage}}=0.5$ indicates that transition reliability and edge connectivity matter more than raw node quantity. The best regime is $\alpha_{\text{usage}}\in[0.3,0.4]$, where the codebook balances coverage with kinematically meaningful prototypes and well-supported transitions.

\textbf{Additional Ablations.} In Appendix~\ref{sec:edge_transition}, Table~\ref{tab:connector_ablation} shows that G2SD is robust to graph edge-weight choices and compares diffusion bridges with geometric connectors (e.g., interpolators, splines). In particular, we observed geometric connectors fail to produce dynamically-feasible locomotion, confirming the need for data-supported local diffusion bridges.

\begin{table}[t]
\centering
\small
\resizebox{0.5\linewidth}{!}{%
\begin{tabular}{c c c c}
\toprule
$\alpha_{\text{usage}}$ & Valid Nodes & Success Rate & Safety-SPEC \\
\midrule
0.00 & 20  & 85\%  & 0.653 \\
0.10 & 100 & 91\%  & 0.814 \\
0.20 & 85  & 91\%  & 0.794 \\
0.30 & 70  & 97\%  & 0.747 \\
\textbf{0.40} & 60 & \textbf{98\%} & 0.770 \\
0.50 & 61  & 87\%  & 0.779 \\
0.60 & 52  & 94\%  & 0.876 \\
0.70 & 43  & 87\%  & 0.518 \\
0.80 & 44  & 88\%  & 0.885 \\
\bottomrule
\end{tabular}%
}
\vspace{2mm}
\caption{\textbf{Ablation on VQ-VAE usage weight.}
Performance is non-monotonic in $\alpha_{\text{usage}}$; the best regime is around $[0.3,0.4]$, where codebook usage yields meaningful nodes and edges.}
\label{tab:alpha_usage}
\vspace{-2mm}
\end{table}

\section{Limitations and Future Work}
G2SD is designed for navigation and locomotion, where similar robot states have similar feasible next transitions. In manipulation, robot configuration, object state, contact mode, and task context should not be abstracted at the same flat level, since similar robot states may have fundamentally different transition structures. A flat state-centric VQ graph may therefore merge contexts with incompatible outgoing transitions. Extending G2SD to manipulation may require a hierarchical or context-aware graph representation that preserves such task-dependent information.

G2SD currently uses Dijkstra search to select a single route for efficiency and reliability, limiting route-level diversity. Future work could use \(k\)-best or stochastic graph search to enable controllable route-level multimodality beyond the single-path setting studied here.

Finally, short conditioned bridges reduce segment-level violation probability under our bound, but do not guarantee $100\%$ safety under all obstacle configurations. Since reactive corrections inside the diffusion planner can reintroduce manifold rupture and produce dynamically infeasible trajectories, safety can be strengthened through execution-time safety filters (e.g., CBF). 
Combining G2SD with such controllers is a natural direction for future work.

\section{Conclusion}
\label{sec:conclusion}
This paper showed that reactive inference-time safety gradients in long-horizon diffusion planning can break dynamical feasibility through manifold rupture. To address this, we proposed \textbf{Graph-Guided Safe Diffuser (G2SD)}, a hierarchical planner that enforces safety structurally via high-level planning over a safety-pruned latent transition graph and uses conditioned diffusion to generate short bridges between consecutive graph-grounded states while remaining anchored to data-supported transitions. Our theoretical analyses explained the failures of reactive safety guidance and showed that decomposing long-horizon generation into anchored segments exponentially reduces safety violation probability. Experiments across navigation and locomotion benchmarks demonstrate that G2SD achieves substantially higher success rates, larger safety margins, and lower inference cost than existing safe diffusion baselines. 

\clearpage
\bibliography{example}

@article{yu2024trajectory,
  title={Trajectory diffusion for objectgoal navigation},
  author={Yu, Xinyao and Zhang, Sixian and Song, Xinhang and Qin, Xiaorong and Jiang, Shuqiang},
  journal={Advances in Neural Information Processing Systems},
  volume={37},
  pages={110388--110411},
  year={2024}
}

@article{carvalho2025motion,
  title={Motion planning diffusion: Learning and adapting robot motion planning with diffusion models},
  author={Carvalho, Joao and Le, An T and Kicki, Piotr and Koert, Dorothea and Peters, Jan},
  journal={IEEE Transactions on Robotics},
  year={2025},
  publisher={IEEE}
}

@inproceedings{seo2025presto,
  title={Presto: Fast motion planning using diffusion models based on key-configuration environment representation},
  author={Seo, Mingyo and Cho, Yoonyoung and Sung, Yoonchang and Stone, Peter and Zhu, Yuke and Kim, Beomjoon},
  booktitle={2025 IEEE International Conference on Robotics and Automation (ICRA)},
  pages={10861--10867},
  year={2025},
}

@inproceedings{xiao2023safediffuser,
  title={{SafeDiffuser}: Safe planning with diffusion probabilistic models},
  author={Xiao, Wei and Wang, Tsun-Hsuan and Gan, Chuang and Hasani, Ramin and Lechner, Mathias and Rus, Daniela},
  booktitle={The Eleventh International Conference on Learning Representations},
  year={2023}
}

@inproceedings{yang2025safeflowmatcher,
  title={{SafeFlowMatcher}: Safe and Fast Planning using Flow Matching with Control Barrier Functions},
  author={Yang, Jeongyong and Jang, Seunghwan and Han, SooJean},
  booktitle={The Fourteenth International Conference on Learning Representations},
  year={2026},
  note={To appear}
}

@inproceedings{mizuta2024cobl,
  title={{CoBL-Diffusion}: {D}iffusion-based conditional robot planning in dynamic environments using control barrier and {L}yapunov functions},
  author={Mizuta, Kazuki and Leung, Karen},
  booktitle={2024 IEEE/RSJ International Conference on Intelligent Robots and Systems (IROS)},
  pages={13801--13808},
  year={2024},
}

@inproceedings{chengsafe,
  title={Safe and Stable Control via {L}yapunov-Guided Diffusion Models},
  author={Cheng, Xiaoyuan and Tang, Xiaohang and Yang, Yiming},
  booktitle={The Thirty-ninth Annual Conference on Neural Information Processing Systems},
  year={2025}
}

@inproceedings{chen2024simple,
  title={Simple Hierarchical Planning with Diffusion},
  author={Chen, Chang and Deng, Fei and Kawaguchi, Kenji and Gulcehre, Caglar and Ahn, Sungjin},
  booktitle={The Twelfth International Conference on Learning Representations}
}

@inproceedings{kaelbling2011hpn,
  title     = {Hierarchical Task and Motion Planning in the Now},
  author    = {Kaelbling, Leslie Pack and Lozano-P{\'e}rez, Tom{\'a}s},
  booktitle = {2011 IEEE International Conference on Robotics and Automation (ICRA)},
  year      = {2011},
  pages     = {1470--1477},
  doi       = {10.1109/ICRA.2011.5980391}
}

@article{garrett2021tamp,
  title   = {Integrated Task and Motion Planning},
  author  = {Garrett, Caelan Reed and Holladay, Rachel and Kaelbling, Leslie Pack and Lozano-P{\'e}rez, Tom{\'a}s},
  journal = {Annual Review of Control, Robotics, and Autonomous Systems},
  year    = {2021},
  volume  = {4},
  pages   = {265--293},
  doi     = {10.1146/annurev-control-091420-084139}
}

@inproceedings{seker2019cnmp,
  title     = {Conditional Neural Movement Primitives},
  author    = {Seker, Mert and Imre, Mert and Piater, Justus and {\v{S}}ar{\'\i}{\'c}, Jo{\v{s}}ko},
  booktitle = {Robotics: Science and Systems (RSS)},
  year      = {2019},
  doi       = {10.15607/RSS.2019.XV.071}
}

@article{lavalle2001rrt,
author = {Lavalle, Steven and Kuffner, James},
year = {2000},
month = {01},
pages = {},
title = {Rapidly-Exploring Random Trees: Progress and Prospects},
journal = {Algorithmic and computational robotics: New directions}
}

@article{Dijkstra1959,
  author  = {Dijkstra, Edsger W.},
  title   = {A note on two problems in connexion with graphs},
  journal = {Numerische Mathematik},
  year    = {1959},
  volume  = {1},
  pages   = {269--271},
  doi     = {10.1007/BF01386390}
}

@inproceedings{blochliger2018topomap,
  title     = {{TopoMap}: Topological Mapping and Navigation Based on Visual {SLAM} Maps},
  author    = {Bl{\"o}chliger, Fabian and Fehr, Marius and Dymczyk, Marcin and Schneider, Thomas and Siegwart, Roland},
  booktitle = {2018 IEEE International Conference on Robotics and Automation (ICRA)},
  year      = {2018},
  doi       = {10.1109/ICRA.2018.8460641}
}

@inproceedings{qureshi2019mpnet,
  title     = {Motion Planning Networks},
  author    = {Qureshi, Ahmed H. and Bency, Mayur J. and Yip, Michael C.},
  booktitle = {2019 International Conference on Robotics and Automation (ICRA)},
  year      = {2019},
  pages     = {2118--2124},
  doi       = {10.1109/ICRA.2019.8793889}
}

@article{saroya2021roadmap,
  title   = {Roadmap Learning for Probabilistic Occupancy Maps with Topology-Informed Growing Neural Gas},
  author  = {Saroya, Manan and Best, Graeme and Hollinger, Geoffrey A.},
  journal = {IEEE Robotics and Automation Letters},
  year    = {2021},
  volume  = {6},
  number  = {3},
  pages   = {4805--4812},
}

@inproceedings{eysenbach2019sorb,
  title     = {Search on the Replay Buffer: Bridging Planning and Reinforcement Learning},
  author    = {Eysenbach, Benjamin and Levine, Sergey},
  booktitle = {Advances in Neural Information Processing Systems},
  year      = {2019},
  volume    = {32},
  pages     = {15246--15257}
}

@inproceedings{kolmogorov_cdf,
  title={Sulla determinazione emp{\'i}rica di uma legge di distribuzione},
  author={A. Kolmogorov},
  year={1933},
  booktitle={Biblioteca Nazionale di Roma}
}

@inproceedings{janner2022diffuser,
  title     = {Planning with Diffusion for Flexible Behavior Synthesis},
  author    = {Janner, Michael and Du, Yilun and Tenenbaum, Joshua B. and Levine, Sergey},
  booktitle = {Proceedings of the 39th International Conference on Machine Learning (ICML)},
  year      = {2022},
}

@inproceedings{wang2022diffusion,
  title         = {Diffusion Policies as an Expressive Policy Class for Offline Reinforcement Learning},
  author        = {Wang, Zhendong and Hunt, Jonathan J. and Zhou, Mingyuan},
  booktitle     = {International Conference on Learning Representations (ICLR)},
  year          = {2023},
  eprint        = {2208.06193},
  archivePrefix = {arXiv},
}

@article{chi2025diffusion,
  title={Diffusion policy: Visuomotor policy learning via action diffusion},
  author={Chi, Cheng and Xu, Zhenjia and Feng, Siyuan and Cousineau, Eric and Du, Yilun and Burchfiel, Benjamin and Tedrake, Russ and Song, Shuran},
  journal={The International Journal of Robotics Research},
  volume={44},
  number={10-11},
  pages={1684--1704},
  year={2025},
  publisher={Sage Publications Sage UK: London, England}
}

@article{dhariwal2021diffusion,
  title={Diffusion models beat gans on image synthesis},
  author={Dhariwal, Prafulla and Nichol, Alexander},
  journal={Advances in neural information processing systems},
  volume={34},
  pages={8780--8794},
  year={2021}
}

@inproceedings{oord2017vqvae,
  title         = {Neural Discrete Representation Learning},
  author        = {van den Oord, A{\"a}ron and Vinyals, Oriol and Kavukcuoglu, Koray},
  booktitle     = {Advances in Neural Information Processing Systems (NeurIPS)},
  year          = {2017},
  eprint        = {1711.00937},
  archivePrefix = {arXiv},
}

@article{brockman2016openai,
  title={Openai gym},
  author={Brockman, Greg and Cheung, Vicki and Pettersson, Ludwig and Schneider, Jonas and Schulman, John and Tang, Jie and Zaremba, Wojciech},
  journal={arXiv preprint arXiv:1606.01540},
  year={2016}
}

@article{fu2020d4rl,
  title={D4rl: Datasets for deep data-driven reinforcement learning},
  author={Fu, Justin and Kumar, Aviral and Nachum, Ofir and Tucker, George and Levine, Sergey},
  journal={arXiv preprint arXiv:2004.07219},
  year={2020}
}

\clearpage
\appendix
\section{Task-Specific Cost Formulation}
\label{app:Cost_Formulation}

In Section~\ref{subsec:onlineplanning}, we formulated the edge weight $W(z_i, z_j)$ as a combination of a task-specific geometric cost $c(z_i, z_j)$ and the transition cost. This appendix details how we define $c(z_i, z_j)$ to incorporate domain objectives.

\textit{Maze2D Navigation: Distance-Based Cost.}
For navigation, the primary objective is to reach the goal efficiently. Since the transition probability term $-\log P(z_j|z_i)$ already penalizes traversing low-density (rare) edges, we define the geometric cost simply as the Euclidean distance:
\begin{equation}
    c_{\text{nav}}(z_i, z_j) \coloneqq \| D(z_i) - D(z_j) \|_2,
\end{equation}
where function $D$ is a decoder that translates the node to its representative position in state space.
Minimizing this cost encourages the planner to find the shortest geometric path, while the weighted transition probability term encourages the path to remain on the high-density data manifold.

\textit{Locomotion: Value and Density-Aware Cost.}
For locomotion tasks (e.g., Hopper, Walker2D), we construct a cost that balances expected return with data support density. $V_{\text{RL}}(s)$ estimates the expected return via offline RL, while $N(z)$ denotes the empirical visitation count of node $z$ in the training dataset, which serves as a density estimate for the data manifold.

We define the node utility $V(z)$ as a composite score:
\begin{equation}
    V(z) \coloneqq \underbrace{V_{\text{RL}}(D_\psi(z))}_{\text{Return}} + \underbrace{\beta \log N(z)}_{\text{Density}}.
\end{equation}
The geometric cost $c_{\text{loco}}(z_i, z_j)$ is then formulated to favor transitions to nodes with higher utility:
\begin{equation}
    c_{\text{loco}}(z_i, z_j) \coloneqq -  \left( V(z_j) - V(z_i) \right),
\end{equation}
where $\beta > 0$ is a weighting coefficients.
By minimizing this cost (maximizing utility gain), the planner is guided towards states that are both high RL value and high visitation count, explicitly preferring the high-density regions of the expert manifold.

\section{Experiment Details}
\label{app:exp_details}

\subsection{Environment Setup}
\textit{Maze2D Navigation.} We use the \texttt{maze2d-large-v1} environment from D4RL. The state space is 4-dimensional, consisting of position $(x, y)$ and velocity $(\dot{x}, \dot{y})$. The action space is 2-dimensional continuous control.

For safety evaluation, we introduce a circular unsafe region centered at $(x_c, y_c) = (6.5, 1.5)$ with radius $r = 1.0$. Episodes are sampled with random start and goal positions, subject to: (i) avoiding wall collisions, (ii) maintaining clearance from the unsafe region. We evaluate on 100 episodes with identical random seeds across all methods.

\textit{Locomotion.} 
We use OpenAI Gym walker2d-medium-expert-v2 and hopper-medium-expert-v2 environments from MuJoCo. The Walker2D agent operates within a 17-dimensional state space governed by a 6-dimensional action space, whereas Hopper involves an 11-dimensional state space and a 3-dimensional action space.

For safety, we set the roof limit to $h_r = 1.4$ for Walker2D and $h_r = 1.6$ for Hopper to prevent the agent from jumping dangerously high or experiencing excessive vertical instability.



\subsection{Performance Metrics}

\textit{Safety-SPEC.}
To strictly measure constraint satisfaction, we compute the minimum value of the safety margin over the entire trajectory $\boldsymbol{\tau}$ of length $H$. This metric captures the worst-case safety margin:
\begin{equation}
\text{Safety-SPEC}(\boldsymbol{\tau}) = \min_{h=0 \dots H-1}\phi(\mathbf{s}_h),
\end{equation}
where $\phi(\cdot)$ denotes the signed safety margin function (positive if safe).
The agent satisfies the safety constraint when Safety-SPEC is positive and violates the constraint if negative.


\textit{Success Rate \& Failure Decomposition.}
This metric reflects the execution performance on maze2d experiments.
An episode is considered a success only if the agent reaches the goal within the time limit without any safety violations and collisions with the maze walls. We define the success rate as:
\begin{equation}
\text{Success} = \frac{1}{E} \sum_{i=1}^{E} \mathbb{I}\left( \text{Goal}_i \land \neg \text{Fail}_i \right),
\end{equation}
where $E$ is the number of episodes (100 tested for main experiment), and failures are decomposed into wall collision ($\text{Fail}_{\text{wall}}$), unsafe-region penetration ($\text{Fail}_{\text{obst}}$), and timeout ($\text{Fail}_{\text{time}}$).

\subsection{Baseline Setups} \label{app:baseline}
We reproduced all baselines in-house by closely following the official implementations and evaluation protocols released by prior work. In particular, we used the public codebases of Diffuser~\citep{janner2022diffuser}, SafeDiffuser~\citep{xiao2023safediffuser}, and Simple Hierarchical Diffusion (SHD)~\citep{chen2024simple} as references, and aligned preprocessing and inference settings with the respective repositories to ensure a fair comparison:

\noindent\textbf{Diffuser:} \url{https://github.com/jannerm/diffuser}\\
\textbf{SafeDiffuser:} \url{https://github.com/Weixy21/SafeDiffuser}\\
\textbf{SHD:} \url{https://github.com/changchencc/Simple-Hierarchical-Planning-with-Diffusion}

For SafeDiffuser, we used the QP-based solver and the ROS configuration provided in the authors' code. Our implementations of Truncated~\citep{brockman2016openai} and Classifier Guidance (CG)~\citep{dhariwal2021diffusion} follow the SafeDiffuser codebase, adopting the same inference-time safety operations to ensure consistent constraint handling across methods.

\subsection{Experimental Hyperparameters}\label{app:hyperparameters}

\providecommand{\tbd}{\textcolor{red}{\texttt{TBD}}}

Maze2D experiments were run on an NVIDIA RTX 4060 Ti GPU with 8GB VRAM and Intel i7-14700K CPU; locomotion experiments ran on an AMD EPYC 9354 CPU and NVIDIA RTX 4090 GPU with 24GB VRAM.

Tables~\ref{tab:app_hparams} and~\ref{tab:app_train_hparams} summarize the hyperparameters used across all experiments. For high-dimensional locomotion tasks (Hopper and Walker2D), the inherent complexity of the state space encourages diverse codebook utilization, rendering the auxiliary usage loss ($\alpha_{\text{usage}}$) unnecessary.
For the low-level planner, we directly used the pretrained weights from SafeDiffuser~\citep{xiao2023safediffuser} with a horizon of $N=32$.

\begin{table}[h]
\centering
\small
\begin{tabularx}{0.80\linewidth}{l X c c c}
\toprule
Component & Hyperparameter & Maze2D & Hopper & Walker2D \\
\midrule
VQ-VAE & Codebook size $Z$ & 256 & 256 & 512 \\
VQ-VAE & Latent dimension $D$ & 16 & 16 & 16 \\
VQ-VAE & Hidden dimension & 128 & 128 & 128 \\
VQ-VAE & Commitment weight $\alpha_{\text{commit}}$ & 0.25 & 0.25 & 0.25 \\
VQ-VAE & Usage weight $\alpha_{\text{usage}}$ & 0.4 & - & - \\
\addlinespace
Graph planner & Edge weights $(\lambda_{\text{prob}},\lambda_{\text{dist}})$ & $(1.0,0.2)$ & (1.0,0.1) & (1.0,0.1) \\
\addlinespace
Diffusion & Denoising steps $T_{\text{diff}}$ & 256 & 20 & 20 \\
Diffusion & Segment length $N$ (steps) & 32 & 32 & 20 \\
Value guidance & Discount factor & - & 0.99 & 0.99 \\
Value guidance & Guidance scale / weight & - & 0.1 & 0.001 \\
\bottomrule
\end{tabularx}
\vspace{1mm}
\caption{Design hyperparameters.}
\label{tab:app_hparams}
\end{table}

\begin{table}[h]
\centering
\small
\begin{tabularx}{0.85\linewidth}{l l X l l l}
\toprule
Module & Dataset & Optimizer & Learning rate & Batch size & Epochs \\
\midrule
\multirow{3}{*}{VQ-VAE} 
  & Maze2D & Adam & $1 \times 10^{-4}$ & 256 & 50 \\
  & Hopper & Adam & $1 \times 10^{-3} \to 1 \times 10^{-4}$ \textsuperscript{\dag} & 512 & 120 \\
  & Walker2D & Adam & $1 \times 10^{-3} \to 1 \times 10^{-4}$ \textsuperscript{\dag} & 512 & 120 \\
\midrule
\multirow{1}{*}{Diffusion} 
  & Maze2D & Adam & $2 \times 10^{-4}$ & 256 & 200 \\
\bottomrule
\multicolumn{6}{l}{\footnotesize \textsuperscript{\dag} Step decay: LR is reduced to $5 \times 10^{-4}$ at epoch 60 and $1 \times 10^{-4}$ at epoch 90.}
\end{tabularx}
\vspace{1mm}
\caption{Training hyperparameters.}
\label{tab:app_train_hparams}
\end{table}

\section{Implementation Details of Online Inference}
\label{app:inference_details}
In this section, we detail the extension of G2SD to receding horizon high-dimensional control (i.e., online inference).
We specifically address the challenges of execution drift where the realized state deviates from the model's prediction during online recomputation.

\textit{High Level Planning for Safety Recovery.}
First, we map the current state to the nearest discrete node $z_{curr}$ in the total graph $\mathcal{G}_\text{total}$. However, a fundamental challenge in high-dimensional control is the inevitability of execution drift. Due to intrinsic stochasticity in system dynamics or unmodeled environmental perturbations, the robot may naturally deviate into states mapping to unsafe nodes ($z_{curr} \notin \mathcal{G}_\text{safe}$).
Since these nodes are pruned from the planning subgraph, standard graph search would fail to initialize.
To resolve this, we trigger a recovery mechanism to identify a valid re-entry point $z_{start} \in \mathcal{G}_\text{safe}$.
We search for immediate outgoing neighbors of $z_{curr}$ in the original total graph $\mathcal{G}_\text{total}$ that reside within the safe subgraph $\mathcal{G}_\text{safe}$, selecting the highest-scoring candidate to prioritize kinematically feasible transitions (i.e., observed in the dataset).
Once a valid start node is established, we perform Dijkstra's algorithm on $\mathcal{G}_\text{safe}$ to find the optimal node sequence $\boldsymbol{\tau}_\text{graph}^*$ (the high-level plan) that minimizes the cumulative edge weight defined in~\eqref{eq:edge_weight}.

\textit{Hybrid Diffusion Stitching.}
To enable real-time receding horizon planning and control (RHC) for locomotion experiments, we propose a hybrid stitching strategy that minimizes computational cost.
A fundamental challenge in diffusion-based control is that reverse denoising is computationally expensive.
Since RHC only executes the first action of the planned trajectory and omits the rest at each cycle, regenerating the full long-horizon path from scratch is highly redundant.
We decouple the generation process into an online immediate segment and a cached future tail. Specifically, the first transition connecting the current state $\mathbf{s}_{curr}$ to the first subgoal $z_1$ is generated in real-time using the conditional diffusion model, acting as a feedback bridge to steer the robot back to the manifold. Conversely, the subsequent path ($z_1 \to z_2 \dots$) utilizes pre-computed trajectories from the graph edges $\mathcal{E}$.
Crucially, these edge trajectories are retrieved from cached solutions generated in prior planning iterations, rather than being recomputed at every control cycle.
This approach maintains high-frequency feedback control via the online bridge while ensuring long-horizon consistency through the cached look-ahead at minimal computation.


\section{VQ-VAE Training Details}
\label{app:vqvae_details}

\textit{Training Objective}.
A VQ-VAE architecture~\citep{oord2017vqvae} consists of an encoder $E_\phi$, a decoder $D_\psi$, and a learnable codebook $\mathcal{C} = \{e_z\}_{z=1}^Z \subset \mathbb{R}^D$, where $Z$ is the codebook size. The training objective comprises four terms: a reconstruction loss for accurate state encoding, a codebook loss to update codebook entries, a commitment loss to anchor encoder outputs to codebook entries, and a usage loss to promote uniform codebook utilization. The total loss $\mathcal{L}_{\text{total}}$ is defined as:
\begin{equation}
    \mathcal{L}_{\text{total}} = \mathcal{L}_{\text{recon}} + \mathcal{L}_{\text{codebook}} + \alpha_{\text{commit}} \mathcal{L}_{\text{commit}} + \alpha_{\text{usage}} \mathcal{L}_{\text{usage}},
\end{equation}
where $\alpha_{\text{commit}}$ and $\alpha_{\text{usage}}$ are hyperparameters weighting the commitment and usage terms, respectively.

\textit{Reconstruction Loss ($\mathcal{L}_{\text{recon}}$).}
This term ensures that the discretized latent representation retains the essential kinematic information of the input state $\mathbf{s}$. It is computed as the mean squared error between the input and the reconstructed state:
\begin{equation}
    \mathcal{L}_{\text{recon}} = \| \mathbf{s} - D_\psi(z_q) \|_2^2,
\end{equation}
where $z_q$ is the quantized latent vector from the codebook.

\textit{Codebook Loss ($\mathcal{L}_{\text{codebook}}$).}
To update the codebook entries, we use an embedding loss that pulls codebook vectors toward the encoder outputs:
\begin{equation}
    \mathcal{L}_{\text{codebook}} = \| z_q - \text{sg}[z_e] \|_2^2,
\end{equation}
where $z_e = E_\phi(s)$ is the encoder output and $\text{sg}[\cdot]$ denotes the stop-gradient operator.

\textit{Commitment Loss ($\mathcal{L}_{\text{commit}}$)}.
To prevent the encoder output from fluctuating arbitrarily, we use the standard commitment loss which pulls the encoder embeddings towards the nearest codebook entries:
\begin{equation}
    \mathcal{L}_{\text{commit}} = \| z_e - \text{sg}[z_q] \|_2^2.
\end{equation}

\textit{Usage Regularization ($\mathcal{L}_{\text{usage}}$)}.
To address the codebook collapse problem where only a small subset of codes is utilized, we introduce a usage regularization term. We first compute a soft assignment distribution over codebook entries using a temperature parameter $\tau_\text{temp}$:
\begin{equation}
    p_{iz} = \frac{\exp(-\|z_e^{(i)} - e_z\|^2 / \tau_\text{temp})}{\sum_{z'=1}^{Z} \exp(-\|z_e^{(i)} - e_{z'}\|^2 / \tau_\text{temp})},
\end{equation}
where $z_e^{(i)}$ is the encoder output for the $i$-th sample in a batch. Let $\bar{q} \in \mathbb{R}^Z$ be the batch-averaged distribution, i.e., $\bar{q}_z = \frac{1}{B} \sum_{i=1}^{B} p_{iz}$. We define $\mathcal{L}_{\text{usage}}$ as the KL divergence from the empirical distribution to a uniform prior:
\begin{equation}
    \mathcal{L}_{\text{usage}} = D_{\text{KL}}(\bar{q} \,\|\, \mathcal{U}) = \sum_{z=1}^{Z} \bar{q}_z \log(Z \cdot \bar{q}_z).
\end{equation}
Minimizing this term encourages the model to utilize all codebook entries uniformly.

\section{Empirical Support for the Theoretical Assumptions}
\label{app:assumption_support}

\textbf{Dominance condition in Theorem~\ref{thm:rupture}.}
The condition in~\eqref{eq:dominance_condition} is a sufficient condition for manifold rupture, but it is difficult to verify exactly in practice because it depends on the manifold-normal component of the safety force and the curvature term. We therefore use a practical surrogate on constraint-active states, defined as states where the safety mechanism produces a nonzero correction. Specifically, we compare the full safety-force magnitude with the local score magnitude to test whether safety corrections dominate the learned score when constraints are active. On Maze2D-large-v1, over 100 episodes, every episode contains at least one constraint-active state where the safety force exceeds the score magnitude. The median safety-to-score ratio is $78.7\times$, with a maximum exceeding $4000\times$, often persisting across multiple denoising steps. This suggests that the dominance regime described in Theorem~\ref{thm:rupture} is a practical concern once safety constraints become active.

\textbf{Conditional independence and Brownian-bridge approximation.}
Assumption~\ref{assump:independence} consists of two modeling components: conditional independence across conditioned segments, and a Brownian-bridge approximation for the stochastic deviation within each segment. In our implementation, the high-level planner specifies the endpoint conditions for each low-level segment, and segment generation is performed using the standard Diffuser inpainting mechanism~\citep{janner2022diffuser}. At every sampling step, the start and end states of each segment are overwritten with the corresponding graph grounded states, while only the intermediate states are updated by the learned score. Each segment is sampled separately with independent denoising noise and shares no intermediate states with neighboring segments, making the segments conditionally decoupled given the graph node sequence.

On Maze2D-large-v1, we evaluate these approximations by subtracting the endpoint-conditioned mean trajectory and computing Pearson correlations between residuals from each segment and those from its adjacent segment. The resulting adjacent-segment residuals exhibit near-zero correlation, with mean \(-0.010\), mean absolute value \(0.054\), and \(83.9\%\) of pairs below \(0.1\), comparable to a permutation baseline. Under identical conditioning, repeated samples also exhibit a parabolic variance profile: variance is near zero at the fixed endpoints and largest in the segment interior, consistent with Brownian-bridge behavior. The mean agreement with the Brownian-bridge parabolic profile is $0.890$, with $92\%$ of segments above $0.8$. Together, these measurements are consistent with the approximations used in our theory and suggest that they provide a reasonable description of the conditioned sampler.

\begin{figure}[h]
    \centering
    \includegraphics[width=0.7\linewidth]{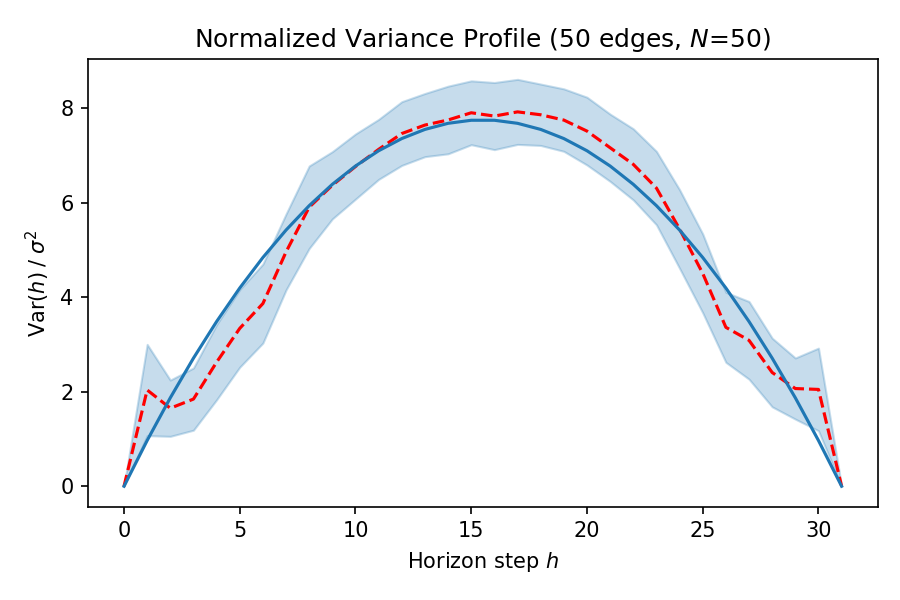}
    \caption{\textbf{Brownian-bridge-like variance profile.}
    For conditioned diffusion segments, the normalized residual variance is small near the fixed endpoints and peaks in the interior, consistent with Brownian-bridge behavior. The blue curve shows the empirical normalized variance averaged over 50 graph edges with \(N=50\) samples per edge, the shaded region indicates variation across edges, and the red dashed curve denotes the Brownian-bridge reference profile.}
    \label{fig:placeholder}
\end{figure}
\section{Additional Ablation Study} \label{sec:edge_transition}

\paragraph{Ablation on edge weights.}
This ablation is tested on the Maze2D navigation task of Section~\ref{sec:exp:maze2d} to study whether high-level planning must explicitly favor \textit{data-supported} edges (via $-\log P(z_j | z_i)$) to avoid selecting brittle graph-node transitions that the low-level diffuser cannot reliably realize. We evaluate the edge scoring function in~\eqref{eq:edge_weight}, which linearly combines Euclidean distance (weighted by $\lambda_{\text{dist}}$) and negative log transition probability (weighted by $\lambda_{\text{prob}}$).

As shown in Table~\ref{tab:ablation_edge_weight}, when $\lambda_{\mathrm{prob}}=0$ (distance-only), the graph search becomes insensitive to how well each edge is supported by the offline trajectories, so it can preferentially select geometrically convenient edges even if their empirical transition support is weak.
In our runs, this setting yields a clear drop in execution reliability (80\% Success Rate; 20\% Crash$\coloneqq 1-\text{Success Rate}$).
Once $\lambda_{\mathrm{prob}}>0$, the planner consistently maintains high Success Rate (97--98\%) across a wide range of $\lambda_{\mathrm{dist}}$, suggesting that incorporating transition likelihood is the key factor for robustness.
Meanwhile, the distance weight mainly affects secondary efficiency metrics (Trajectory Steps and planning Time) only.

\begin{table}[t]
\centering
\small
\setlength{\tabcolsep}{4.5pt}
\renewcommand{\arraystretch}{1.08}
\resizebox{0.8\columnwidth}{!}{%
\begin{tabular}{cc|cc|c|cc}
\toprule
$\lambda_{\mathrm{prob}}$ & $\lambda_{\mathrm{dist}}$ &
Success Rate (\%) & Crash (\%) &
Safety-SPEC $(\uparrow)$ & Trajectory Steps $(\downarrow$) &
Time (s) $(\downarrow)$ \\
\midrule
0.0 & 1.0 & 80.0 & 20.0 & 0.913 & 168.8 $\pm$ 128.8 & 0.95 \\
\midrule
0.5 & 0.5 & 97.0 & 3.0  & 0.770 & 219.8 $\pm$ 138.3 & 1.03 \\
1.0 & 0.0 & 98.0 & 2.0  & 0.770 & 211.0 $\pm$ 140.0 & 1.02 \\
1.0 & 0.2 & 98.0 & 2.0  & 0.770 & 219.5 $\pm$ 138.2 & 1.03 \\
1.0 & 0.5 & 98.0 & 2.0  & 0.770 & 210.3 $\pm$ 138.9 & 1.00 \\
2.0 & 0.2 & 98.0 & 2.0  & 0.770 & 219.5 $\pm$ 138.9 & 1.09 \\
\midrule
1.0 & 1.0 & 97.0 & 3.0  & 0.770 & 207.2 $\pm$ 138.4 & 1.00 \\
1.0 & 2.0 & 97.0 & 3.0  & 0.770 & 206.3 $\pm$ 138.1 & 0.98 \\
1.0 & 5.0 & 95.0 & 5.0  & 0.770 & 205.6 $\pm$ 137.6 & 0.97 \\
\bottomrule
\end{tabular}%
}
\vspace{1mm}
\caption{\textbf{Ablation on edge transition weights.} We vary the edge cost weights $(\lambda_{\mathrm{prob}},\lambda_{\mathrm{dist}})$ and report execution success rate, safety-SPEC, and efficiency metrics on the Maze2D navigation task.
A positive $\lambda_{\mathrm{prob}}$ is clearly required for path safety and reliability, while $\lambda_{\mathrm{dist}}$ mainly affects secondary efficiency metrics.
Safety-SPEC is nearly unchanged across different weights.
}
\label{tab:ablation_edge_weight}
\end{table}

\begin{table}[t]
\centering
\resizebox{0.8\columnwidth}{!}{
\begin{tabular}{lcccc}
\toprule
Config \(\beta\) &
Walker2D Score &
Walker2D Safety &
Hopper Score &
Hopper Safety \\
\midrule
\(0.0\) &
0.4945 & 0.2660 & \textbf{0.4636} & 0.1308 \\
\(0.1\) (default) &
\textbf{0.4963} & \textbf{0.2972} & 0.4460 & \textbf{0.1310} \\
\(50.0\) &
0.4872 & 0.1789 & 0.4626 & 0.1205 \\
\bottomrule
\end{tabular}
}
\vspace{1mm}
\caption{Ablation of the return--density trade-off under tighter locomotion
constraints. The default setting \(0.1\) achieves the best Walker2D score and safety, and the best Hopper safety.}
\label{tab:locomotion_return_density_ablation}
\end{table}

\paragraph{Ablation on the geometric cost.}
We further evaluate the sensitivity of G2SD to the return--density trade-off on Walker2D and Hopper under more challenging constraint settings, with lower obstacle heights: \(1.3\) for Walker2D and \(1.5\) for Hopper. We vary
the configuration \(\beta\), where \(\beta\) weights the density-support term in the high-level graph objective.

As shown in Table~\ref{tab:locomotion_return_density_ablation}, the results show that the return term is the primary driver of task performance, while the density-support term acts as a complementary regularizer. Compared with the return-only setting \((\beta=0.0)\), adding a small positive density weight \((\beta=0.1)\) improves Safety-SPEC on both Walker2D and Hopper. This suggests that a moderate density term helps steer the graph search toward transitions that are better supported by the offline data and safer under tighter constraints.

The ablation also shows that overly emphasizing density support is not beneficial. Because the return and density terms have different numerical scales, setting \(\beta=50.0\) makes the density term dominate the graph objective. While a small density weight improves safety, a large density weight reduces Walker2D safety and slightly degrades overall performance. Overall, G2SD is robust across a range of edge-weight configurations, but the best trade-off is achieved when the return objective is retained, and the density term is used as a moderate support regularizer.

\paragraph{Comparison with geometric connectors.}

We further evaluate whether the diffusion bridge can be replaced by simple geometric connectors between selected graph states. This experiment is conducted on Walker2D and Hopper under more challenging constraint settings, with lower roof limits of \(1.3\) for Walker2D and \(1.5\) for Hopper. We compare our diffusion-based connector against linear interpolation, cubic Hermite spline, zero-velocity Hermite interpolation, and minimum-snap interpolation.

As shown in Table~\ref{tab:connector_ablation}, geometric connectors achieve positive safety margins but yield substantially lower locomotion scores than the diffusion bridge. On Walker2D, the best non-diffusion connector reaches only \(0.024\) score, while the diffusion bridge achieves \(0.4963\), corresponding to a roughly \(20\times\) improvement. On Hopper, the best non-diffusion connector obtains \(0.057\), whereas the diffusion bridge reaches \(0.4460\), giving a roughly \(8\times\) improvement.

These results show that satisfying geometric safety constraints between endpoints is not sufficient for underactuated locomotion. 
Although interpolation-based connectors are computationally cheaper, they do not explicitly model the contact- and dynamics-dependent structure required for executable walking or hopping. 
In contrast, the diffusion bridge samples intermediate motions from the learned local motion manifold, producing segments that better preserve data-supported locomotion dynamics.

\begin{table}[t]
\centering
\resizebox{1.0\linewidth}{!}{
\begin{tabular}{lccccc}
\toprule
Connector 
& Walker2D Score 
& Walker2D Safety 
& Hopper Score 
& Hopper Safety 
& Avg. Time (s) \\
\midrule
Diffusion (Ours)        & \textbf{0.4963} & 0.2972 & \textbf{0.4460} & 0.1310 & 0.560 \\
Linear Interpolation    & -0.002          & 0.4534 & 0.008           & 0.4579 & \textbf{0.002} \\
Cubic Hermite Spline    & 0.011           & 0.4604 & 0.057           & 0.4355 & \textbf{0.002} \\
Hermite (zero-velocity) & 0.024           & \textbf{0.4640} & 0.016 & \textbf{0.4579} & 0.003 \\
Minimum Snap            & 0.003           & 0.4630 & 0.045           & 0.3801 & 0.003 \\
\bottomrule
\end{tabular}
}
\vspace{1mm}
\caption{
Comparison between diffusion bridges and geometric connectors on Walker2D and Hopper under tighter roof constraints.
Although geometric connectors often yield positive safety margins, their task scores remain near zero, showing that safe endpoint connection alone does not produce dynamically executable locomotion.
The diffusion bridge achieves substantially higher scores by generating data-supported intermediate motions.
}
\label{tab:connector_ablation}
\end{table}


\end{document}